\documentclass[10pt]{article} 

\usepackage[preprint]{rlj} 

\usepackage{amssymb}            
\usepackage{mathtools}          
\usepackage{mathrsfs}           
\usepackage{graphicx}           
\usepackage{subcaption}         
\usepackage[space]{grffile}     
\usepackage{url}                
\usepackage{pgfplots}
\usepackage{amsthm}
\usepackage{wrapfig}
\usepackage{booktabs}

\pgfplotsset{compat=1.18}
\newtheorem{definition}{Definition}
\newtheorem{proposition}{Proposition}
\newtheorem{theorem}{Theorem}
\newtheorem{lemma}{Lemma}

\title{On the Structural Non-Preservation of Epistemic Behaviour under Policy Transformation}

\setrunningtitle{On the Structural Non-Preservation of Epistemic Behaviour under Policy Transformation}

\author{Alexander Galozy\textsuperscript{1}}

\emails{alexander.galozy@hh.se, alex.galozy@gmail.com}

\affiliations{
$^{1}$\textbf{Center for Applied Intelligent Systems Research, Halmstad University, Sweden}\\
}

\contribution{
We formalise epistemic behaviour as probe-conditioned behavioural dependency: variation in action distributions with respect to internally accumulated information under fixed observations, defined at the level of observable policy behaviour without assuming access to architectural or latent structure.
}
{
Internal memory and belief are central to RL under partial observability \citep{sutton1998introduction,ghavamzadeh2015bayesian} and are exploited in meta-RL and world-model agents \citep{duan2016rl2,ha2018worldmodels}. Prior work typically treats these mechanisms architecturally. We instead characterise information-conditioned interaction directly through probe-induced response profiles.
}

\contribution{
We define probe-relative $\epsilon$-behavioural equivalence and a within-policy behavioural distance that quantifies probe-conditioned interaction patterns.
}
{
Trajectory-level comparisons need not distinguish policies that differ in internal information processing. Related notions such as bisimulation focus on state abstractions and value consistency \citep{feres2019bisimulation,Zang2023bisimulationCollapse}, while multi-agent policy-distance metrics compare action distributions across agents at matched observations \citep{bettini2025snd,Hu2024mapd}. Our formulation defines equivalence over probe-induced response profiles within a single policy, without assuming identifiable latent states.
}

\contribution{
We prove via Proposition~1 that the set of policies exhibiting non-trivial behavioural dependency is not closed under convex aggregation. We further prove a sufficient local condition (Theorem~1) under which gradient ascent on a skewed mixture objective decreases behavioural distance when the dominant-mode gradient aligns with the direction of steepest contraction.
}
{
Convex aggregation linearly interpolates response profiles and contracts behavioural distance, with explicit constructions demonstrating non-closure. Theorem~1 isolates a geometric alignment condition under which reward-driven optimisation locally reduces probe-conditioned separation under skewed latent priors. Minimal partially observable MDP experiments provide controlled witnesses of these mechanisms.
}


\summary{
Many reinforcement learning (RL) agents rely on internally accumulated information to guide action selection under partial observability. Yet common RL pipelines transform policies through aggregation and reward-driven optimisation without explicitly preserving such information-conditioned interaction patterns. We formalise epistemic behaviour as behavioural dependency, defined as systematic variation in action selection with respect to internal information under fixed observations, and introduce probe-relative behavioural equivalence. We prove that policies exhibiting non-trivial behavioural dependency are not closed under convex aggregation, and prove a sufficient local condition under which gradient-based optimisation contracts behavioural distance under distributional imbalance. These results provide a structural account of when information-conditioned behaviour is retained or attenuated under policy transformation.
}
\begin{document}

\makeCover  
\maketitle  

\begin{abstract}
Reinforcement learning (RL) agents under partial observability often condition actions on internally accumulated information such as memory or inferred latent context. We formalise such information-conditioned interaction patterns as \emph{behavioural dependency}: variation in action selection with respect to internal information under fixed observations. This induces a probe-relative notion of $\epsilon$-behavioural equivalence and a within-policy behavioural distance that quantifies probe sensitivity. We establish three structural results. First, the set of policies exhibiting non-trivial behavioural dependency is not closed under convex aggregation. Second, behavioural distance contracts under convex combination. Third, we prove a sufficient local condition under which gradient ascent on a skewed mixture objective decreases behavioural distance when a dominant-mode gradient aligns with the direction of steepest contraction. Minimal bandit and partially observable gridworld experiments provide controlled witnesses of these mechanisms. In the examined settings, behavioural distance decreases under convex aggregation and under continued optimisation with skewed latent priors, and in these experiments it precedes degradation under latent prior shift. These results identify structural conditions under which probe-conditioned behavioural separation is not preserved under common policy transformations.
\end{abstract}

\section{Introduction}

Reinforcement learning (RL) is typically evaluated through observable trajectories and task performance. Many behaviours that enable effective learning, however, depend on internally accumulated information beyond immediate observations. Under partial observability, policies often condition actions on memory or inferred latent context, even when external observations are fixed. Such information-conditioned behaviour is central to meta-RL \citep{duan2016rl2,rakelly2019pearl}, Bayesian approaches \citep{zintgraf2021variBad}, world models \citep{ha2018worldmodels,hafner2019dreamerv1,schrittwieser2020muzero}, and structured exploration \citep{pathak2017curiosity,bellemare2016cts,houthooft2016vime}.

Standard RL evaluation through aggregate reward, however, does not distinguish whether performance arises from robust information-conditioned strategies or from adaptation to dominant training modes. A policy may maintain high reward on frequently sampled contexts while losing probe-conditioned distinctions required for robustness under distributional shift. This creates an evaluation gap: reward stability does not imply preservation of internal dependencies that enable generalisation to rare or shifted contexts. While previous foundational works employ auxiliary self-supervised objectives \citep{jaderberg2017unreal,schwarzer2020spr} and regularisation techniques \citep{haarnoja2018sac} to maintain representational structure, such interventions lack formal characterisation of what epistemic structure is preserved.

Common RL transformations, such as convex aggregation, gradient-based optimisation and distillation, operate on parameters or action distributions without explicit constraints on information-conditioned interaction patterns. Because such patterns are encoded implicitly within representations, standard pipelines offer no inherent guarantee of their retention. Concrete manifestations in meta-RL, world models, and exploration are discussed in Appendix~\ref{app:vulnerabilities}.

To analyse when epistemic structure is preserved under policy transformation, we formalise epistemic behaviour as \textit{behavioural dependency}: variation in action selection with respect to internal information under fixed observations. Quantifying this variation yields a within-policy behavioural distance that, under a binary latent-mode assumption, upper-bounds worst-case return under prior shift and induces a probe-relative notion of $\epsilon$-behavioural equivalence over policies.

We establish that epistemic structure is not preserved under common transformations in the examined settings. First, we prove that policies exhibiting behavioural dependency are not closed under convex aggregation. Second, we prove a sufficient local condition under which gradient-based optimisation under skewed priors contracts behavioural distance, even while short-horizon reward remains near-optimal under the dominant prior.

\textbf{Behavioural Representation and Preservation (BRP)} articulates the representational perspective implied by these results, treating $\epsilon$-behavioural equivalence classes as the relevant units of preservation. The following sections develop the formal machinery, provide minimal reinforcement learning witnesses, and analyse the structural challenges of preserving epistemic behaviour.

\section{Beyond Kinematic Behaviour}
\label{sec:BeyondKinematics}

Standard reinforcement learning evaluates agents at the level of observable trajectories and reward. However, identical trajectories can arise from distinct internal information-processing strategies. To analyse preservation of information-conditioned behaviour, we distinguish observable interaction from the internal states that modulate action selection.

Kinematic behaviour refers to observable action and state sequences that determine reward and trajectory statistics. Latent behaviour refers to internal representations such as recurrent memory states or inferred task variables that mediate the mapping from observations to actions.

We define \emph{epistemic behaviour} as systematic variation in action selection with respect to internally accumulated information under fixed observations. This notion abstracts away from architectural details and focuses directly on information-conditioned interaction patterns.

Behavioural preservation concerns whether such probe-conditioned dependencies remain stable under transformations such as optimisation, compression, or aggregation. Because these dependencies are typically encoded implicitly within parameters, standard reinforcement learning pipelines do not in general provide structural guarantees of their retention. Preservation therefore requires analysing policies at the level of behavioural equivalence classes instead of individual parameters or latent states. We now formalise behavioural dependency and define the probe-relative equivalence classes that serve as the units of preservation.

\subsection{Behavioural Dependency and Equivalence}
\label{sec:behavioural_equivalence}

To analyse preservation of epistemic behaviour, we formalise behavioural dependency and the equivalence classes it induces. Since internal representations are architecture dependent and not directly identifiable, equivalence must be defined at the level of observable interaction instead of parameters. Policy-level behavioural distances based on action distributions at fixed observations have been proposed in multi-agent settings \citep{bettini2025snd,Hu2024mapd}, but those approaches measure differences across agents, whereas our focus is probe-induced variation within a single policy.

\paragraph{Notation.}
Let $\pi$ denote a stochastic policy over observation space $\mathcal{O}$ and action space $\mathcal{A}$. The internal state $h_t \in \mathcal{H}_\pi$ represents any mechanism by which a policy accumulates information over time. No assumption is made about the interpretability or identifiability of $h_t$. All equivalence notions are defined relative to a probe set $\mathcal{P}$. A policy interacting with an environment samples actions as
\[
a_t \sim \pi(a_t \mid o_t, h_t).
\]

\paragraph{Behavioural Dependency and Response Profiles.}
A policy exhibits epistemic behaviour if its action distribution varies as a function of internal information under fixed observations. Formally, behavioural dependency holds if there exist $o \in \mathcal{O}$ and $h, h' \in \mathcal{H}_\pi$ such that $\pi(\cdot \mid o, h) \neq \pi(\cdot \mid o, h')$.

Let $\mathcal{P}$ denote a set of interaction probes. For a policy $\pi$, define its response profile as
\begin{equation}
R_\pi(p,o) = \mathbb{E}_{h \sim p(\pi)}\!\left[\pi(\cdot \mid o,h)\right],
\end{equation}
which captures the observable action distribution under probe $p \in \mathcal{P}$ at observation $o$. For fixed $(p,o)$, $R_\pi(p,o)$ lies in the probability simplex $\Delta(\mathcal{A})$.

\paragraph{Behavioural Equivalence.}

To allow graded comparison of probe-conditioned response profiles, we define a probe-relative divergence at a fixed evaluation observation $o^\ast$:
\begin{equation}
d_{\mathcal{P}, o^\ast}(\pi_1, \pi_2)
=
\sup_{p \in \mathcal{P}}
\left\|
R_{\pi_1}(p,o^\ast)
-
R_{\pi_2}(p,o^\ast)
\right\|_1.
\end{equation}

\begin{definition}[$\epsilon$-Behavioural Equivalence]
Two policies $\pi_1$ and $\pi_2$ are said to be $\epsilon$-behaviourally equivalent with respect to $\mathcal{P}$ at $o^\ast$ if
\begin{equation}
d_{\mathcal{P}, o^\ast}(\pi_1, \pi_2) \le \epsilon.
\end{equation}
\end{definition}

The corresponding equivalence class is
\begin{equation}
[\pi]_{\mathcal{P},\epsilon}
=
\left\{
\pi' \mid
d_{\mathcal{P}, o^\ast}(\pi,\pi') \le \epsilon
\right\}.
\end{equation}

The equivalence class $[\pi]_{\mathcal{P}, \epsilon}$ groups policies with sufficiently similar probe-conditioned response profiles. We note that the term behavioural equivalence has appeared previously in Interactive POMDPs to aggregate models of other agents that induce identical optimal actions \citep{rathnasabapathy2006exact}. Our definition instead compares policies directly through probe-conditioned response profiles, independent of optimality criteria or reward structure.

\paragraph{Within-Policy Behavioural Distance.}

Fix an evaluation observation $o^\ast$ and two probes $p_0, p_1 \in \mathcal{P}$ chosen to induce distinct latent interaction contexts. The within-policy behavioural distance is defined as
\begin{equation}
d(\pi)
=
\left\|
R_\pi(p_0, o^\ast)
-
R_\pi(p_1, o^\ast)
\right\|_1.
\end{equation}

A policy exhibits non-trivial behavioural dependency (with respect to the selected probes) if $d(\pi) > 0$. This scalar quantity isolates probe-induced variation in action distributions at a fixed evaluation context and serves as a minimal observable surrogate for information-conditioned behavioural structure.

\subsection{The Functional Implication of Behavioural Distance}
\label{sec:functional_implication}

We next connect behavioural distance to task-level performance. In partially observable tasks where distinct latent modes require distinct optimal actions, insufficient probe-conditioned separation imposes a structural limit on robustness under prior shift. Let latent modes $m \in \{0,1\}$ require distinct optimal actions $a_0^\ast \neq a_1^\ast$ at $o^\ast$, and let $J(\pi \mid P)$ denote expected return under prior $P(m)$. Let $h_m$ denote the internal state induced at evaluation by the probe corresponding to latent mode $m$. Assume the expected return under a given mode $m$ is bounded by $V_{\max} - \Delta(1 - \pi(a_m^\ast \mid o^\ast, h_m))$, meaning any suboptimal action incurs a value penalty of at least $\Delta > 0$.

\begin{lemma}[Robustness Requires Epistemic Distance]
\label{lem:robustness_bound}
If the value penalty assumption holds and $d(\pi) \le \epsilon$, then
\begin{equation}
\min_{P} J(\pi \mid P) \le V_{\max} - \frac{\Delta}{2}(1 - \epsilon).
\end{equation}
\end{lemma}

\begin{proof}
Let $\pi_{m}(a)=\pi(a \mid o^\ast, h_{m})$. Since $a_{0}^\ast \ne a_{1}^\ast$, the action probabilities under mode 1 satisfy $\pi_{1}(a_{0}^\ast)+\pi_{1}(a_{1}^\ast)\le 1$. The $L_1$ behavioural distance constraint $d(\pi) \le \epsilon$ implies that for any single action $a$, $|\pi_{0}(a)-\pi_{1}(a)|\le \epsilon$. Consequently, for the optimal action $a_0^\ast$, we have $\pi_1(a_0^\ast) \ge \pi_0(a_0^\ast) - \epsilon$.

Substituting this lower bound into the sum constraint yields
\begin{equation}
\pi_0(a_0^\ast) - \epsilon + \pi_1(a_1^\ast) \le 1 
\;\Rightarrow\;
\pi_0(a_0^\ast) + \pi_1(a_1^\ast) \le 1 + \epsilon.
\end{equation}

Using $\min(x,y) \le \frac{x+y}{2}$ gives
\begin{equation}
\min(\pi_0(a_0^\ast), \pi_1(a_1^\ast)) \le \frac{1}{2}(1 + \epsilon).
\end{equation}

Substituting into the value model yields
\begin{equation}
\min_{P} J(\pi \mid P)
\le
V_{\max} - \Delta\!\left(1 - \frac{1+\epsilon}{2}\right)
=
V_{\max} - \frac{\Delta}{2}(1 - \epsilon),
\end{equation}
completing the proof.
\end{proof}

Lemma~\ref{lem:robustness_bound} shows that behavioural distance is not merely descriptive but functionally necessary for robustness under latent prior shift. When $d(\pi)$ is small, the policy cannot simultaneously allocate high probability mass to both mode-specific optimal actions, imposing a ceiling on worst-case return. In the limiting case $d(\pi)=0$, the bound reduces to $\min_P J(\pi\mid P) \le V_{\max} - \Delta/2$.

\subsection{Structural Non-Preservation under Policy Transformation}
\label{sec:structural_fragility_theory}

We now analyse how common policy transformations affect behavioural distance.

\paragraph{Convex Aggregation.}
For $\pi_\alpha = \alpha \pi_1 + (1-\alpha)\pi_2$, linearity implies
\begin{equation}
R_{\pi_\alpha}(p,o^\ast) = \alpha R_{\pi_1}(p,o^\ast) + (1-\alpha)R_{\pi_2}(p,o^\ast).
\end{equation}

\begin{lemma}[Convex Contraction]
\label{lem:convex_contraction}
For $\alpha \in [0,1]$,
\begin{equation}
d(\pi_\alpha) \le \alpha d(\pi_1) + (1-\alpha)d(\pi_2).
\end{equation}
\end{lemma}

\begin{proof}
Let $\Delta_i = R_{\pi_i}(p_0,o^\ast) - R_{\pi_i}(p_1,o^\ast)$. Then
\begin{equation}
R_{\pi_\alpha}(p_0,o^\ast) - R_{\pi_\alpha}(p_1,o^\ast)
=
\alpha \Delta_1 + (1-\alpha)\Delta_2.
\end{equation}
Applying the triangle inequality yields the claim.
\end{proof}

Define $\mathcal{E} = \{\pi \mid d(\pi) > 0\}$, where $d(\pi)$ is defined with respect to the fixed probe pair $(p_0,p_1)$.

\begin{proposition}[Non-Closure under Convex Aggregation]
\label{prop:nonclosure_epistemic}
The set $\mathcal{E}$ is not closed under convex combination.
\end{proposition}

\begin{proof}
Let $\pi_1 \in \mathcal{E}$ with distinct response distributions $q_0 = R_{\pi_1}(p_0, o^\ast)$ and $q_1 = R_{\pi_1}(p_1, o^\ast)$. Construct $\pi_2$ by swapping these responses so that $R_{\pi_2}(p_0, o^\ast)=q_1$ and $R_{\pi_2}(p_1, o^\ast)=q_0$. Then $d(\pi_2)>0$. Their differences are anti-aligned:
\begin{equation}
R_{\pi_1}(p_0,o^\ast) - R_{\pi_1}(p_1,o^\ast)
=
-
\big(
R_{\pi_2}(p_0,o^\ast) - R_{\pi_2}(p_1,o^\ast)
\big).
\end{equation}
For $\alpha=\tfrac12$, $d(\pi_{1/2})=0$, so $\pi_{1/2}\notin\mathcal{E}$.
\end{proof}

\paragraph{Gradient Transformation.}

We now analyse how gradient-based transformations affect behavioural distance. The following result is conditional and local. It does not assert that reward optimisation universally decreases behavioural distance; rather, it isolates a sufficient geometric condition under which contraction must occur.

\begin{theorem}[Conditional Local Contraction Under Gradient Alignment]
\label{thm:gradient_erosion}
For parameters $\theta$, assume $\nabla_{\theta}d(\pi_{\theta})\ne0$ and define
\[
v_{d}
=
-\frac{\nabla_{\theta}d(\pi_{\theta})}
{\|\nabla_{\theta}d(\pi_{\theta})\|_{2}}.
\]
Let $J(\theta)=(1-\delta)J_{0}(\theta)+\delta J_{1}(\theta)$ with $\delta\in(0,0.5)$. 

If $\nabla J_{0}(\theta)^{\top}v_{d}\ge k>0$ and $\|\nabla J_{1}(\theta)\|_{2}\le L$, then whenever
\[
\delta<\frac{k}{k+L},
\]
we have $\nabla J(\theta)^{\top}v_{d}>0$, implying local decrease of $d(\pi_\theta)$ under gradient ascent.
\end{theorem}

\begin{proof}
\[
\nabla J(\theta)^{\top}v_{d}
=
(1-\delta)\nabla J_{0}(\theta)^{\top}v_{d}
+
\delta\nabla J_{1}(\theta)^{\top}v_{d}.
\]
Using the alignment bound $\nabla J_0(\theta)^\top v_d \ge k$ and Cauchy-Schwarz with $\|v_d\|_2=1$, we have $\nabla J_1(\theta)^\top v_d \le \|\nabla J_1(\theta)\|_2 \le L$. Therefore:
\[
\nabla J(\theta)^{\top}v_{d}
\ge
(1-\delta)k - \delta L,
\]
which is positive when $(1-\delta)k > \delta L$, i.e., when $\delta<\frac{k}{k+L}$.
\end{proof}

The theorem provides a sufficient local condition under which reward optimisation contracts behavioural distance. When the alignment condition holds under skewed priors, gradient ascent reduces probe-conditioned separation even if dominant-prior reward remains stable. Whether such alignment arises depends on task structure, representation geometry, and optimisation dynamics.

\section{Structural Non-Preservation: Formal and Empirical Witnesses}
\label{sec:Fragility}

We provide minimal formal constructions and controlled empirical witnesses designed to isolate the structural mechanisms identified in Section~\ref{sec:structural_fragility_theory}. These experiments are not intended as benchmark evaluations, but as diagnostic settings in which the theoretical conditions can be directly measured.

\subsection{Sequential Degradation and Robustness under Prior Shift}
\label{sec:seqrobustprior}

We consider a partially observable gridworld to examine structural degradation in a sequential setting (Figure~\ref{fig:behavioural-sketch}). An agent may execute a locally suboptimal diagnostic probe that reveals a latent mode $m \in \{0,1\}$ determining the rewarding goal location. We evaluate three policies: a Probing policy ($d=2$), a Shortcut policy that ignores the probe ($d=0$), and an Aggregated policy formed by majority voting ($d=1$).

\begin{figure}[t]
    \centering
    \includegraphics[width=\linewidth]{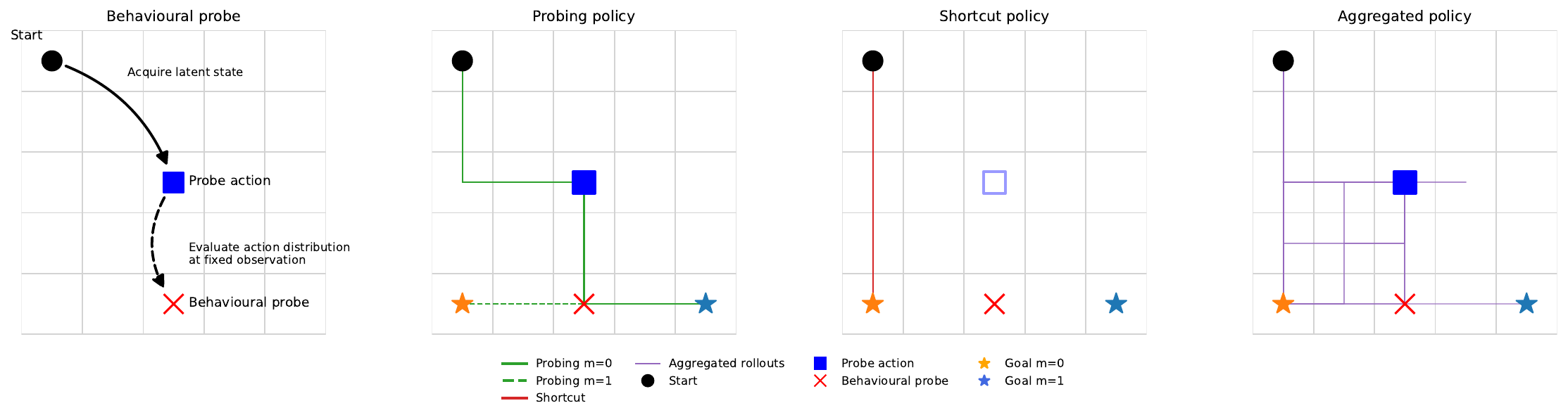}
    \caption{
    Behavioural probe evaluation and representative trajectories in the partially observable gridworld. 
    \textbf{Left:} The evaluation protocol. The agent acquires latent mode information to induce an internal hidden state $h_m$. It is subsequently evaluated at a fixed observation $o^\star$ to measure the conditional policy $\pi(\cdot \mid o^\star, h_m)$. 
    \textbf{Second:} Probing policy. The agent maintains separated internal representations and successfully executes distinct actions contingent on the initial mode. 
    \textbf{Third:} Shortcut policy. The agent consistently navigates toward the same goal irrespective of the latent mode. 
    \textbf{Right:} Aggregated policy. The agent may execute the probing action, but it fails to exhibit the consistent information-conditioned action differences of the probing agent.
    }
    \label{fig:behavioural-sketch}
\end{figure}

To isolate structural robustness, we evaluate these policies under a biased prior $\mathbb{P}(m=0)=0.9$ and a reversed prior $\mathbb{P}(m=0)=0.1$ without additional training. We also evaluate convex mixtures $\pi_\alpha = \alpha \pi_{\text{probe}} + (1-\alpha)\pi_{\text{shortcut}}$ across both priors. Experiments are run over 300 episodes.

\begin{figure}[htbp]
\centering
\begin{minipage}{0.48\linewidth}
\centering
\small
\begin{tabular}{lcc}
Policy & Biased Prior & Reversed Prior \\
\midrule
Probing    & $0.810 \pm 0.000$ & $0.810 \pm 0.000$ \\
Shortcut   & $0.863 \pm 0.017$ & $0.060 \pm 0.019$ \\
Aggregated & $0.710 \pm 0.023$ & $0.000 \pm 0.020$ \\
\end{tabular}
\end{minipage}
\hfill
\begin{minipage}{0.48\linewidth}
\centering
\includegraphics[width=\linewidth]{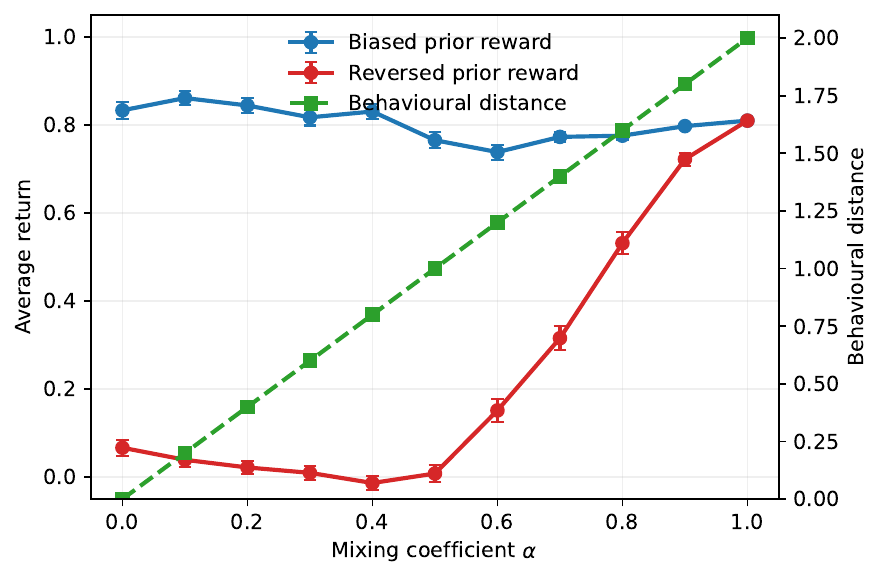}
\end{minipage}
\caption{
\textbf{Robustness under latent prior shift.}
\textbf{Left:} Average return of three policies evaluated under a biased prior and a reversed prior. The Shortcut policy attains high return under the biased prior despite zero behavioural distance, but degrades sharply under prior reversal. The Probing policy remains stable across priors ($mean \pm se$).
\textbf{Right:} Convex mixtures $\pi_\alpha = \alpha \pi_{\text{probe}} + (1-\alpha)\pi_{\text{shortcut}}$. Behavioural distance decreases linearly with $\alpha$, and robustness under prior shift decreases correspondingly, indicating that sensitivity to latent distribution shift is governed by probe-conditioned behavioural separation and not by biased-prior reward alone.
}
\label{fig:prior_shift_combined}
\end{figure}

As shown in Figure~\ref{fig:prior_shift_combined}, the Shortcut policy attains high return under the biased prior despite zero behavioural distance, but degrades sharply under prior reversal. The Probing policy remains stable across priors. Across convex mixtures, robustness under prior shift scales approximately with behavioural distance rather than biased-prior return. This illustrates that reward statistics under a dominant prior do not necessarily reflect structural robustness to latent distribution shift.

\subsection{Optimisation-Induced Structural Erosion}
\label{sec:optErosion}

We examine whether epistemic behavioural structure, once established, is preserved under continued gradient-based optimisation when the training objective becomes distributionally biased. Unlike the aggregation experiments, no explicit policy mixing is applied; the only transformation is reward-driven optimisation under a skewed latent prior.

The task consists of three phases: a probe revealing a latent binary mode $m \in \{0,1\}$, a delay phase with distractors independent of $m$, and an evaluation step at a fixed observation $o^\star$ containing no information about $m$. Correct action at $o^\star$ therefore requires retention of probe information in recurrent state.

Recurrent Advantage Actor–Critic (A2C) policies are first trained under a uniform prior ($P(m=0)=0.5$), yielding probe-conditioned separation. Optimisation then continues under a biased prior ($P(m=0)=0.98$). We measure return under biased and reversed priors, behavioural distance $d(\pi)$ at $o^\star$, hidden-state separation quantified by both $\| h_{m=0} - h_{m=1} \|_2$ and its normalised variant. Further, we measure the \textbf{majority force} ($Proj_0 = \nabla J_0(\theta)^\top \mathbf{v}_d$) and \textbf{minority force} ($Proj_1 = \nabla J_1(\theta)^\top \mathbf{v}_d$) as the projections of mode-specific gradients onto the structural contraction vector $\mathbf{v}_d$, the weighted sum of which constitutes the \textbf{net structural force} $(1 - \delta)Proj_0 + \delta Proj_1$ that determines the direction of representational erosion. Results are averaged over 10 seeds (Figure~\ref{fig:optimisation_erosion}). An extended description of the experimental setup and experimental evaluation over different prior shifts as well as network size, is presented in supplementary material (Appendix~\ref{app: experimental_details}).

Under biased optimisation, return on the dominant prior remains near-optimal while performance under the reversed prior degrades. Behavioural distance $d(\pi)$ contracts during the biased phase and stabilises at a reduced plateau. Absolute and relative hidden-state separation decrease in parallel, indicating deformation of internal representation geometry alongside behavioural attenuation.

\begin{figure}[t]
    \centering
    \includegraphics[width=\linewidth]{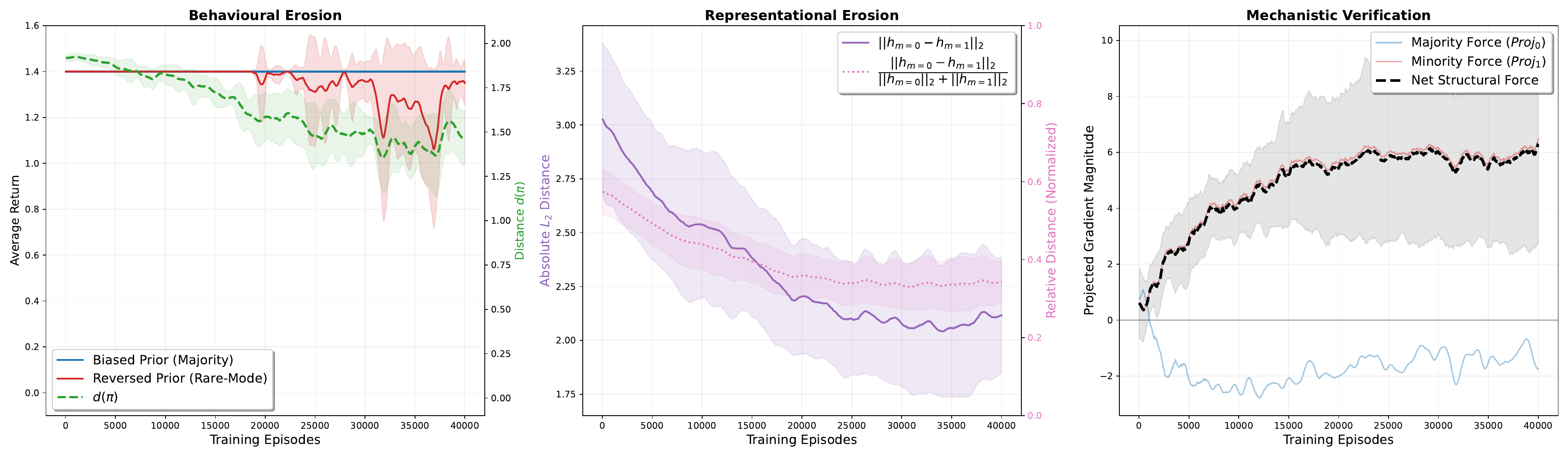}
    \caption{
    \textbf{Optimisation-induced structural erosion under a heavily biased prior (98\%).}
    \textbf{Left: Behavioural erosion.} Return under the dominant prior remains stable while reversed-prior performance degrades. Behavioural distance decreases during biased optimisation and stabilises at a lower value.
    \textbf{Middle: Representational erosion.} Absolute and normalised hidden-state separation contract over training.
    \textbf{Right: Mechanistic verification.} Projection of task gradients onto $\mathbf{v}_d = -\nabla d / \|\nabla d\|$ yields a persistent positive net structural force under the biased prior, consistent with Theorem~\ref{thm:gradient_erosion}. Shaded regions denote $\pm$ one standard deviation across 10 seeds.
    }
    \label{fig:optimisation_erosion}
\end{figure}

The projected structural force remains positive throughout the erosion regime, indicating local alignment between reward optimisation and contraction of behavioural distance. Although contraction slows as $d(\pi)$ decreases, the persistent alignment is consistent with the sufficient condition in Theorem~\ref{thm:gradient_erosion}. 

These results show that stability of reward under a dominant prior does not ensure preservation of probe-conditioned structure. Continued optimisation under skewed objectives can attenuate epistemic behavioural dependency even without explicit aggregation. These observations motivate analysing preservation at the level of probe-conditioned behavioural equivalence classes, which treats them as the units of preservation under policy transformation.

\section{Behavioural Representation and Preservation: An Interpretive Lens}
\label{sec:brp}

The preceding analysis suggests viewing probe-conditioned behavioural equivalence classes as units of structural comparison under policy transformation. We refer to this perspective as Behavioural Representation and Preservation (BRP). BRP is not proposed as a formal framework or algorithmic prescription. It provides a lens for evaluating whether reinforcement learning transformations preserve probe-conditioned behavioural structure.

\subsection{Analytical Dimensions of Preservation}

Under the BRP lens, preservation of probe-conditioned behavioural structure can be analysed along three dimensions. These are descriptive criteria for evaluating how a transformation affects behavioural equivalence classes.

\paragraph{Identifiability.}
Identifiability concerns whether distinct information-conditioned strategies are distinguishable under the chosen probe set $P$. A probe set induces identifiable structure if behaviourally distinct policies exhibit sufficiently separated response profiles (i.e., $d(\pi) \gg 0$ relative to tolerance $\varepsilon$). If probes fail to induce measurable separation, equivalence classes collapse trivially, and preservation cannot be meaningfully assessed. Identifiability therefore depends jointly on probe design and task structure.

\paragraph{Retention.}
Retention concerns whether a transformation $T$ preserves probe-conditioned response profiles. Given a policy $\pi$ and transformed policy $T(\pi)$, retention can be evaluated by measuring whether $\pi$ and $T(\pi)$ remain within the same $\epsilon$-behavioural equivalence class under the chosen probes. As shown in Section~\ref{sec:structural_fragility_theory}, common transformations such as convex aggregation and gradient updates under skewed priors do not generally guarantee such invariance. Retention is thus a property of the transformation relative to a probe set, not of parameter similarity or reward stability.

\paragraph{Transferability.}
Transferability concerns whether probe-conditioned behavioural structure can be reproduced across policies, training regimes, or environments. Two policies trained under different initialisations or optimisation dynamics may belong to the same $\epsilon$-equivalence class even if their internal representations differ. Transferability therefore focuses on reproducibility of response profiles. Under this view, preservation across agents or training phases amounts to re-instantiating the same probe-relative behavioural class.

\subsection{Preservation Criterion}

Under this lens, epistemic competence corresponds to stability of probe-conditioned response profiles under transformation $T$. A transformation is preservation-consistent if it is approximately $\epsilon$-equivalence-preserving with respect to the chosen probe set. BRP therefore provides a structural criterion for evaluating whether reinforcement learning methods retain information-conditioned interaction patterns across optimisation, aggregation, or compression.

BRP does not prescribe a specific optimisation method. Instead, it suggests that claims about preservation of epistemic behaviour should be evaluated at the level of probe-conditioned response profiles and not solely through reward statistics. High reward under a dominant prior is insufficient to establish preservation of behavioural dependency.

\section{Discussion and Conclusion}

This paper analysed the preservation of epistemic behaviour under common reinforcement learning transformations. By formalising probe-relative behavioural dependency and $\epsilon$-behavioural equivalence, we showed that convex aggregation does not preserve non-trivial dependency, and that gradient-based optimisation under skewed priors can locally contract behavioural distance under a sufficient alignment condition. Empirically, controlled experiments verify these mechanisms, demonstrating that behavioural erosion precedes degradation under latent prior shift even when dominant-prior reward remains near-optimal. The observed gradient-induced erosion shares mechanistic similarities with `loss of plasticity' in continual learning \citep{dohare2024, lyle2021understanding}, though our framework formalises this degradation specifically at the level of information-conditioned interaction.

Behavioural Representation and Preservation (BRP) provides an interpretive lens that treats $\epsilon$-equivalence classes as units of structural comparison. Open questions remain, including characterising when the gradient-alignment condition arises across architectures and designing scalable probes in high-dimensional environments. Together, these results introduce a diagnostic vocabulary to evaluate whether future algorithmic interventions successfully retain information-conditioned behaviour.


\appendix

\bibliography{main}
\bibliographystyle{rlj}


\beginSupplementaryMaterials

\section{Extended Discussion: Structural Vulnerabilities in RL Pipelines}
\label{app:vulnerabilities}

Epistemic behaviour arises when action selection depends on internally accumulated information rather than solely on immediate observation. As established in Section \ref{sec:behavioural_equivalence}, because such dependency is not represented explicitly in standard reinforcement learning pipelines, common transformations are not guaranteed to preserve $\epsilon$-behavioural equivalence. 

The mechanisms below illustrate exactly how the mathematical vulnerabilities identified in Section \ref{sec:structural_fragility_theory} (dilution, non-closure, and gradient erosion) actively manifest across modern reinforcement learning paradigms.

\subsection{Implicitness and Optimisation-Induced Collapse in Meta-RL}
Information-conditioned behaviour is typically realised through parameters, latent activations, or recurrent dynamics trained exclusively for reward maximisation. Meta-reinforcement learning systems such as $RL^2$ \citep{duan2016rl2} and latent-context approaches including PEARL and VariBAD \citep{rakelly2019pearl,zintgraf2021variBad} condition actions on accumulated interaction history or inferred task variables, thus instantiating behavioural dependency. 

However, because optimisation updates operate on representations rather than behavioural equivalence classes, these methods are susceptible to the continuous structural erosion described by Theorem~\ref{thm:gradient_erosion}. When the task distribution (latent prior) is skewed or exhibits curriculum shifts, the gradient update $T_{\text{grad}}$ acts as an equivalence-breaking transformation under the stated prior conditions, inducing a tendency toward reduction of behavioural distance even when short-horizon reward remains high. Consequently, information-conditioned distinctions that support rapid adaptation can be attenuated during continued training once the minority tasks are rarely sampled, leading to silent losses in meta-generalisation even if average reward remains stable.

\subsection{Aggregation and Behavioural Homogenisation}
Aggregation procedures, such as parameter averaging, policy distillation, and logit ensembling, operate on parameters or action distributions and do not enforce the preservation of response profiles. World-model-based agents such as Dreamer and MuZero \citep{ha2018worldmodels,hafner2020dreamerv2,schrittwieser2020muzero} rely on latent predictive states or internal search processes to guide action selection. 

When such policies are compressed or aggregated, short-horizon reward performance may remain similar while distinctions in probe-relative response profiles are reduced. In distributed or federated reinforcement learning, periodic averaging treats deviations across agents as noise, even when those deviations reflect distinct information-conditioned strategies. As shown by Lemma~\ref{lem:convex_contraction} (Convex Contraction) and Proposition~\ref{prop:nonclosure_epistemic} (Non-Closure), convex aggregation can break equivalence class membership and cannot increase behavioural distance beyond the weighted average. Because epistemic structure is at most diluted under convex combination, this aggregation systematically homogenises and can eliminate structured epistemic behaviour across the population.

\subsection{Temporal Erosion of Structured Exploration}

Exploration strategies driven by novelty, prediction error, or uncertainty estimates depend on internally accumulated signals and therefore instantiate behavioural dependency \citep{pathak2017curiosity,bellemare2016cts,houthooft2016vime}. Empirical studies report that exploratory behaviour often diminishes as predictive models improve or uncertainty estimates contract \citep{burda2019largecuriosity}.

The gradient mechanism formalised in Theorem~\ref{thm:gradient_erosion} provides one structural explanation for this phenomenon. When optimisation becomes dominated by a subset of modes or states, reward gradients may align with directions that reduce distinctions no longer directly contributing to expected return significantly. In such regimes, the projected gradient onto the contraction direction $\mathbf{v}_d$ can become positive, inducing attenuation of behavioural dependency even while reward remains stable.

We do not claim that exploration collapse is universal or inevitable. The present lens identifies a sufficient structural condition under which internally encoded exploratory distinctions may erode during continued reward-driven optimisation. Absent explicit constraints that preserve $\epsilon$-behavioural equivalence classes, exploration-related representations can be gradually attenuated when they cease to provide immediate reward advantage.

\section{Supplemental Material: Experimental Details}
\label{app: experimental_details}
\subsection{Environment Specification: The Abstract Epistemic MDP}
The Abstract Epistemic Environment is a partially observable MDP designed to force reliance on internal memory by separating critical information from reward-bearing observations. The environment consists of three temporal phases:

\begin{itemize}
\item \textbf{Probe Phase ($t=0$):} A latent binary mode $m \in \{0,1\}$ is revealed via a dedicated bit in the 5D observation vector $[\,\text{is\_probe},\,\text{is\_delay},\,\text{is\_eval},\,\text{distractor},\,\text{mode\_bit}\,]$.
\item \textbf{Delay Phase ($1 \le t < T$):} The agent must match random distractor bits to survive, receiving a $+0.1$ reward per step. Failure to match results in immediate termination ($-1.0$ reward).
\item \textbf{Evaluation Phase ($t=T$):} The agent receives a fixed ``Zero Info'' observation $o^\ast = [0,0,1,0,0]$. It must output the action matching the initial mode $m$ to receive a $+1.0$ reward.
\end{itemize}

The environment is shown schematically in figure~\ref{fig:bottleneck}.

\begin{figure}[h]
\centering
\includegraphics[width=0.5\textwidth]{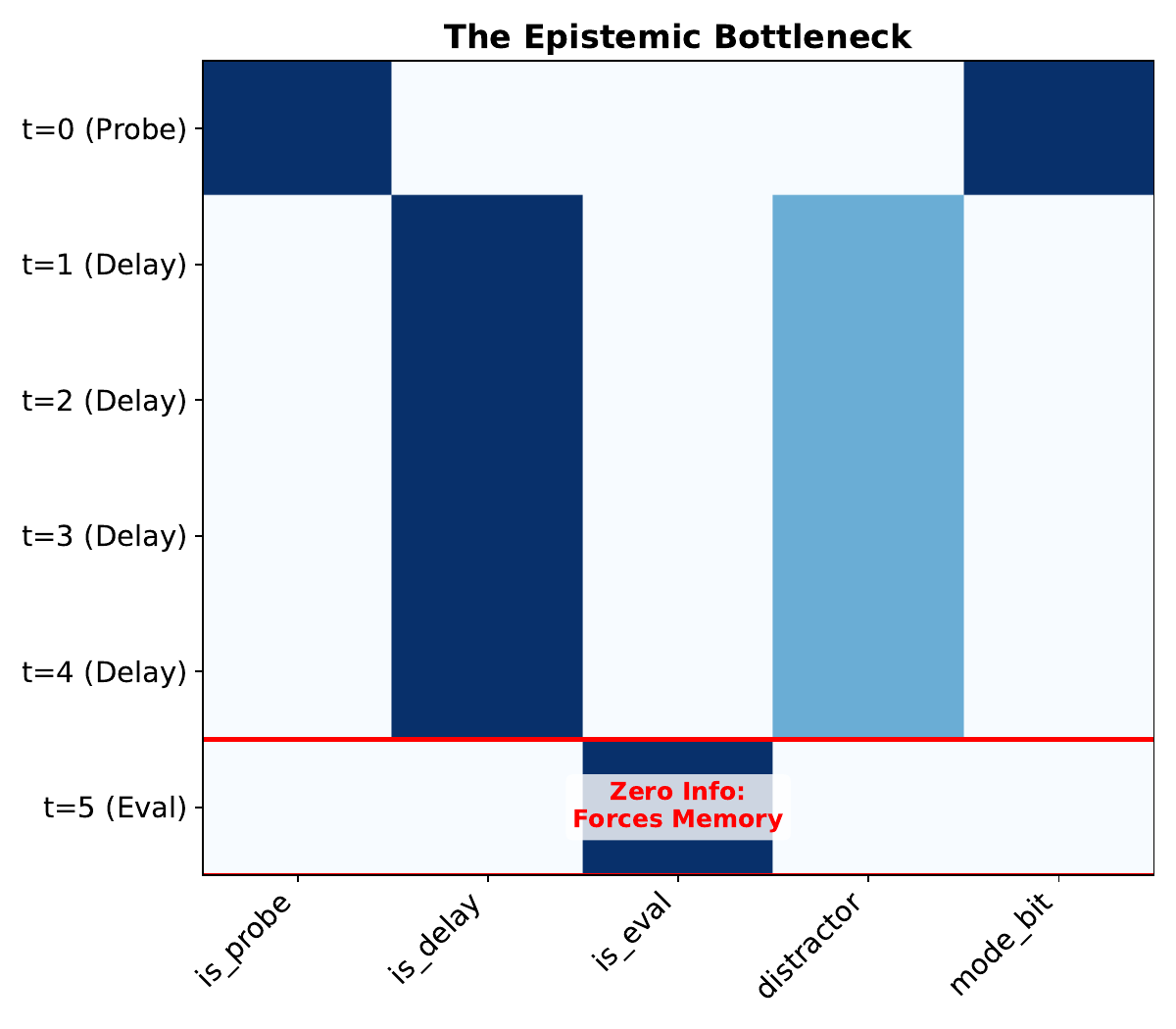}
\caption{
Abstract Epistemic MDP. A three-phase partially observable task consisting of a probe phase revealing latent mode $m$, a distractor delay phase independent of $m$, and a final evaluation phase at a fixed observation $o^\ast$ that requires retention of probe information. The task enforces an epistemic bottleneck between information acquisition and reward.
}
\label{fig:bottleneck}
\end{figure}

\subsection{Optimisation and Hyperparameters}
Structural erosion is observed using a two-stage training protocol. In \textbf{Stage 1 (Synthesis)}, policies are trained under a uniform prior ($\delta=0.5$) using Advantage Actor--Critic (A2C) until the behavioural distance $d(\pi) > 1.90$. To encourage epistemic competence (by recalling the latent mode in the hidden state), the agent receives a reward of $+5$ upon successfully reaching the goal. In \textbf{Stage 2 (Erosion)}, optimisation continues under biased priors ($\delta \in \{0.6,\dots,0.98\}$) and $+1$ reward for reaching the goal.

\begin{table}[h]
\centering
\caption{Hyperparameters for Stage 2 (Erosion) Experiments}
\label{tab:params}
\begin{tabular}{ll}
\hline
\textbf{Parameter} & \textbf{Value} \\
\hline
Optimizer & Adam \\
Learning Rate & $2 \times 10^{-4}$ \\
Weight Decay & $1 \times 10^{-3}$ \\
Batch Size & 32 episodes \\
Discount Factor ($\gamma$) & 0.99 \\
Entropy Coefficient & 0.02 \\
\hline
\end{tabular}
\end{table}

\subsection{Sensitivity to Prior Skew}
Empirical results across 10 independent seeds (Table~\ref{tab:delta_metrics}) indicate a phase transition toward functional degradation as the prior skew exceeds $\delta = 0.95$. At $\delta = 0.965$, behavioural variance increases markedly (0.449), suggesting the onset of structural instability. As shown in Figure~\ref{fig:sensitivity}, behavioural distance decreases systematically with increasing prior skew, and extreme bias ($\delta = 0.98$) is associated with functional degradation and reduced $d(\pi)$. The decay of internal representation geometry precedes observable functional degradation, indicating that structural contraction serves as an early signal of robustness failure.

\begin{table}[h]
\centering
\caption{Final Epistemic Metrics across Prior Skew ($\delta$) Sweep ($mean \pm se$, 10 seeds)}
\label{tab:delta_metrics}
\begin{tabular}{lcc}
\hline
\textbf{Prior Skew ($\delta$)} & \textbf{Final Behavioural Distance $d(\pi)$} & \textbf{Final Hidden Distance $h_{\text{dist}}$} \\
\hline
$\delta=0.600$ & $1.939 \pm 0.010$ & $2.809 \pm 0.153$ \\
$\delta=0.800$ & $1.911 \pm 0.027$ & $2.714 \pm 0.208$ \\
$\delta=0.900$ & $1.853 \pm 0.039$ & $2.582 \pm 0.153$ \\
$\delta=0.950$ & $1.743 \pm 0.092$ & $2.443 \pm 0.141$ \\
$\delta=0.965$ & $1.517 \pm 0.449$ & $2.133 \pm 0.426$ \\
$\delta=0.980$ & $1.369 \pm 0.238$ & $2.050 \pm 0.203$ \\
\hline
\end{tabular}
\end{table}

\begin{figure}[h]
\centering
\includegraphics[width=\textwidth]{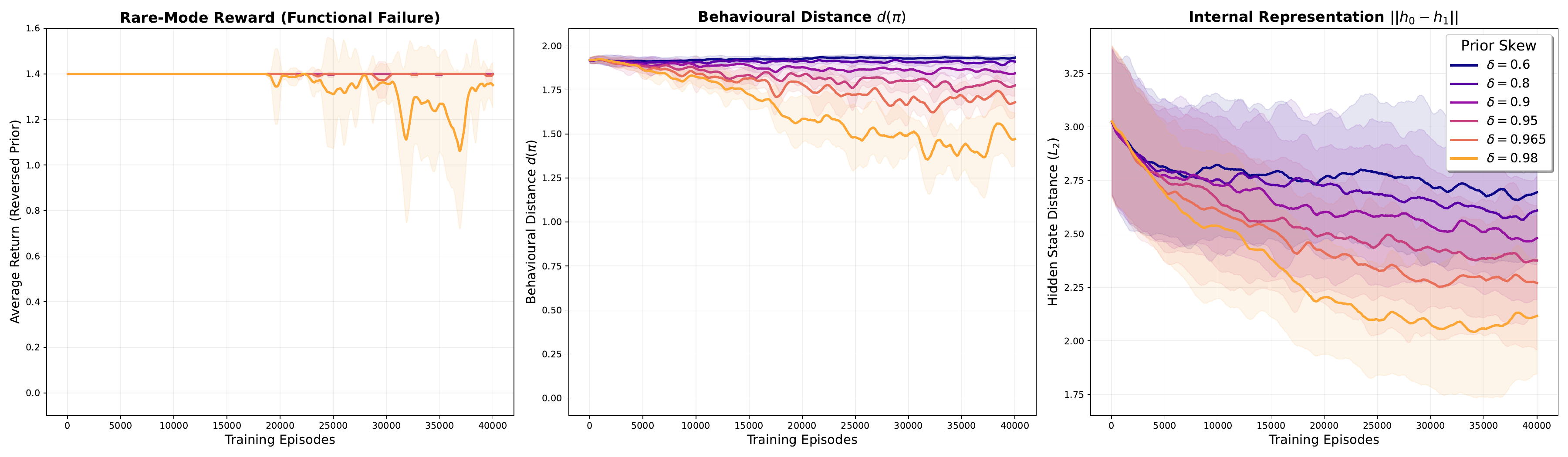}
\caption{\textbf{Sensitivity to Prior Skew.} (Left) Functional degradation on the rare mode most pronounced at $\delta=0.98$. (Center) Behavioural distance $d(\pi)$ decreasing as skew increases. (Right) Decay of internal representation geometry across biased priors, preceding functional failure.}
\label{fig:sensitivity}
\end{figure}

\subsection{Effect of Network Capacity under Extreme Skew}
To test whether erosion is an artifact of limited architectural capacity, we swept the hidden dimension $H \in \{32,64,128\}$ under a fixed extreme bias ($\delta=0.98$). Final values are summarized in Table~\ref{tab:capacity_metrics}.

Increasing recurrent capacity does not prevent behavioural erosion under extreme skew. While the $H=128$ model begins with a larger initial representational separation, it exhibits greater variability during training compared to smaller models (Figure~\ref{fig:capacity_results}).

\begin{table}[h]
\centering
\caption{Final Epistemic Metrics by Network Capacity ($H$) at $\delta=0.98$ ($mean \pm se$, 10 seeds)}
\label{tab:capacity_metrics}
\begin{tabular}{lcc}
\hline
\textbf{Hidden Units ($H$)} & \textbf{Final Behavioural Distance $d(\pi)$} & \textbf{Final Hidden Distance $h_{\text{dist}}$} \\
\hline
$32$  & $1.350 \pm 0.171$ & $2.025 \pm 0.205$ \\
$64$  & $1.422 \pm 0.437$ & $2.357 \pm 0.414$ \\
$128$ & $1.266 \pm 0.618$ & $2.565 \pm 0.607$ \\
\hline
\end{tabular}
\end{table}

\begin{itemize}
\item \textbf{Functional instability:} Higher capacity ($H=128$) induces oscillatory rare-mode reward, consistent with repeated partial recovery and subsequent attenuation of minority representations.
\item \textbf{Capacity-invariant erosion:} Despite a $4\times$ capacity increase, all models converge to similarly reduced behavioural distances ($d(\pi) \approx 1.3$).
\item \textbf{Variance scaling:} Larger networks exhibit higher variance in final $d(\pi)$ without mitigating structural erosion.
\end{itemize}

\begin{figure}[h]
\centering
\includegraphics[width=\textwidth]{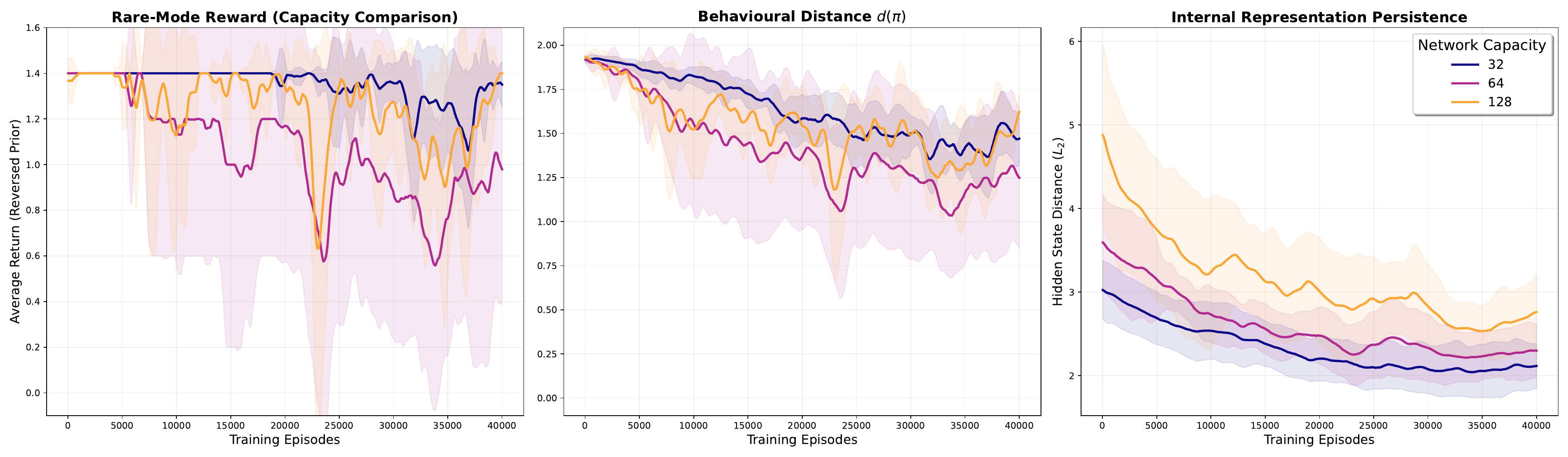}
\caption{\textbf{Capacity Sweep Results.} (Left) Rare-mode reward exhibits increased instability at $H=128$. (Center) Behavioural erosion remains comparable across network sizes. (Right) Representational distance decays across all capacities under extreme skew.}
\label{fig:capacity_results}
\end{figure}

\subsection{Implications for Structural Preservation}

These results empirically witness the theoretical vulnerabilities identified by the BRP lens, particularly for policy gradient methods. In our A2C experiments and in the analysed setting, gradient-driven optimisation under sufficiently skewed priors attenuates probe-conditioned behavioural separation, eroding $d(\pi)$ regardless of architectural capacity. While overparameterised networks can theoretically maintain dual-mode representations, the strong gradient bias actively disfavors this. This echoes recent findings in the continual learning literature where excess capacity fails to prevent loss of plasticity \citep{dohare2024, lyle2021understanding}. Although alternative learning dynamics, such as value-based methods, may exhibit different structural sensitivities, these findings indicate that excess capacity alone yields functional instability rather than structural preservation. This suggests that epistemic competence cannot be passively guaranteed through architecture alone, motivating the need for explicit constraints at the behavioural level.
\end{document}